\relax

\documentclass{article}
\usepackage{ijcai20}

\usepackage{times}
\usepackage{soul}
\usepackage{url}
\usepackage[utf8]{inputenc}
\usepackage[small]{caption}
\usepackage{graphicx}
\usepackage{amsmath}
\usepackage{amsthm}
\usepackage{booktabs}
\usepackage{algorithm}
\usepackage{algorithmic}
\usepackage{color}

\usepackage{bbm}
\usepackage{mathtools}
\usepackage{latexsym}
\usepackage{bm}
\usepackage{amssymb}
\usepackage{amsfonts}
\usepackage{booktabs}
\usepackage{multirow}
\usepackage{bigstrut}
\usepackage{mathtools, nccmath}

\newcommand{\ty}[0]{\tilde{y}}
\newcommand{\heta}[0]{\hat{\eta}}
\newcommand{\ep}[0]{\mathbb{E}}
\newtheorem{lemma}{Lemma}
\newtheorem{theorem}{Theorem}
\newtheorem{remark}{Remark}
\newtheorem{example}{Example}
\DeclarePairedDelimiter{\round}\lfloor\rceil
\DeclarePairedDelimiter\norm{\lVert}{\rVert}

\DeclareMathOperator*{\argmin}{arg\,min}

\title{Collective Loss Function for Positive and Unlabeled Learning}
\author{\Large \textbf{Chenhao Xie\textsuperscript{\rm 1}\textsuperscript{\rm 2}, Qiao Cheng\textsuperscript{\rm 1}, Lihan Chen\textsuperscript{\rm 1}, Jiaqing Liang\textsuperscript{\rm 2}, Yanghua Xiao\textsuperscript{\rm 1}}\\ 
\textsuperscript{\rm 1}Fudan University, Shanghai.\\ 
\textsuperscript{\rm 2}Shuyan Technology, Shanghai.}

\begin{document}

\maketitle

\begin{abstract}
    People learn to discriminate between classes without explicit exposure to negative examples.
    On the contrary, traditional machine learning algorithms often rely on negative examples, otherwise the model would be prone to collapse and always-true predictions.
    Therefore, it is crucial to design the learning objective which leads the model to converge and to perform predictions unbiasedly without explicit negative signals.
    In this paper, we propose a \textbf{\underline{C}}ollectively loss function to learn from only \textbf{\underline{P}}ositive and \textbf{\underline{U}}nlabeled data (cPU).
    We theoretically elicit the loss function from the setting of PU learning.
    We perform intensive experiments on the benchmark and real-world datasets. 
    The results show that cPU consistently outperforms the current state-of-the-art PU learning methods.

\end{abstract}

\section{Introduction}
Positive and unlabeled learning (PU learning) aims at learning from only positive and unlabeled examples, without explicit exposure to negative examples.
This setting arises from multiple practical application scenarios: retrieving information with limited feedback given~\cite{Onoda2005OneCS}, text classification with only positive labels collected~\cite{Yu2003TextCF} and detecting area of interest in images where normal samples are available but the abnormal samples are scarce and diverse~\cite{Zuluaga2011LearningFO,Li2011APA}.
It is widely applicable in industrial scenarios such as content censorship~\cite{Ren2014PositiveUL,Li2014SpottingFR}, disease gene detection~\cite{Yang2012PositiveunlabeledLF} and drug discovery~\cite{Liu2017ComputationalDD}.

Positive labels are considered prefect in most literatures while the unlabeled data are not and thus handled in different ways.
The first category tries to identify negative samples from the unlabeled data and convert the problem back to positive-negative classification~\cite{Liu2002PartiallySC,Li2003LearningTC}.
The heuristic strategies in these methods act as external information to recognize negative samples.
However, these strategies often heavily rely on subtle design for a single task/dataset and results in low transferability.
The second category take the unlabeled data as corrupted negative samples.
Early approaches attempt to reweight the unlabeled data~\cite{Liu2003BuildingTC,Lee2003LearningWP} with a smaller penalty per sample, but their performances are upper bounded due to their intrinsic bias, proved by Du Plessis et al.~\shortcite{Plessis2014AnalysisOL,Plessis2015ConvexFF} who later develop an approach, called uPU, with a non-convex losses to cancel the bias.
This work is extended by nnPU~\cite{Kiryo2017PositiveUnlabeledLW} to avoid overfitting by preventing risk estimators from reaching negative values.
Hou et al.~\shortcite{Hou2017GenerativeAP} further argue that overfitting is still an issue with flexible deep neural networks.
They proposed GenPU, a generative adversarial approach, to address the challenge of limited positive data, whereupon they train two discriminators: one telling fake generated examples from the true and the other assign positive labels to the generated examples that are similar to the positive class. 

A similar keypoint behind all the aforementioned solutions 
is that they all try to \textit{recover the true distribution} of positive and negative data and thus \textit{recover the true risk}.
However, performing risk rectification at the outcome-of-loss-function level, which is the main cause of the inaccuracy, according to our elaboration in section~\ref{sec:risknn}.
In this paper, we propose a novel method called ``\textbf{\underline{C}}ollectively loss function to learn from \textbf{\underline{P}}ositive and \textbf{\underline{U}}nlabeled data" (cPU) to rectify the \textit{predictor} instead of the total \textit{risk}.
We collectively gather predictions from predictors and rectify them before the calculation of loss function.
We design our method with the following principles in mind:
\begin{enumerate}
    \item \textit{Minimum intervention.} The only difference between PU learning setting and regular positive/negative learning is that the negative data are not explicitly labeled. 
    Therefore, we also hope minimum feature construction is required. 
    Hence, we only process at prediction level and leave the feature engineering part to the powerful representation of models (especially neural networks) themselves.
    \item \textit{Robustness.} Due to the class uncertainty in the unlabeled data, it is demanding to estimate the class prior accurately in the unlabeled data.
    As a result, we take the collective prediction to balance the randomness of mini-batch.    
\end{enumerate}



Our main contributions are threefold.
Firstly, we provide a unbiased approach of estimating the posterior probability in PU learning setting, which is in harmony with very flexible models and learns on a large scale.
Secondly, we propose a general framework of studying the behavior of loss functions via elicitation.
Thirdly, we derive the collective loss function to rectify the decision boundary drift and theoretically bounded the generalization error.
We conduct comprehensive experiments of in comparison with state-of-the-art approaches.








\section{Problem Statement}
\label{sec:ps}

Consider the input space $\mathcal{X} \subseteq  \mathbb{R}^d$ and label space $\mathcal{Y}=\{0,1\}$, we denote by $P_{XY}$ the joint distribution over $\mathcal{X}\times \mathcal{Y}$.
Let $f: \mathcal{X} \mapsto \mathbb{R}$ be a \textit{predictor function} and $h(X)=\round{\sigma(f(X))}$ be the \textit{classifier}, where $\round{\cdot}$ means rounding to the nearest integer and $\sigma(\cdot): \mathbb{R} \mapsto (0,1)$ is a sigmoid function. For example, $\sigma(\cdot) = 1/(1+\exp(\cdot))$.
Let $\ell: \mathbb{R}\times\{0,1\} \mapsto \mathbb{R}^+$ be the \textit{loss function} for binary classification.
According to statistical learning theory~\cite{Vapnik1999AnOO} the risk of $f(x)$ is defined by
\begin{equation}
    \label{eq:risk}
    R(f) := P_{XY}( h(X) \neq Y)= \ep_{XY}[ \ell(f(x),y)],
\end{equation}
where $\ep_{XY}[\cdot]=\ep_{(x,y)\sim P_{XY}}[\cdot]$.

\subsection{Conditional Risk for PN Learning.}
Let $\eta(x) := P_{Y|X}(1|x)$ denote the posterior probability of the positive class. 
For clarity, we omit the argument $(x)$ in expressions such as $f(x)$ and $\eta(x)$ throughout the paper.
The risk decomposes to the conditional form
\begin{equation}
    \label{eq:r-cond}
    \begin{split}
        R(\eta,f) &= \ep_{X}[ \ep_{Y|X}[\ell(f,y)|X=x]  ] \\
            &= \ep_{X}[P_{Y|X}(1|x)\ell(f,1)+P_{Y|X}(0|x)\ell(f,0)] \\
            &= \underbrace{\ep_{X}[\eta\ell(f,1)]}_{R_p(\eta,f,1)} + \underbrace{\ep_{X}[(1-\eta)\ell(f,0)}_{R_n(\eta,f,0)}], 
    \end{split}    
\end{equation}
where $\ep_{X}[\cdot]=\ep_{x\sim P(X)}[\cdot]$.
$P(X)$ is the marginal distribution of $X$.
We expect Fisher consistency~\cite{lin2004note} (aka classification-calibration in some literatures) on the predictors, which is a very weak condition.
To be specific, if the risk $R$ in \eqref{eq:r-cond} is minimized, the following equation holds.
\begin{equation}\label{eq:h}
    h^*= \round{\sigma(f^*)} = \round{\eta},
\end{equation}
An example loss function  is the zero-one loss:
\begin{equation}
    \ell_{0/1} = \begin{cases}
       0 & h(x)   =  y \\
       1 & h(x) \neq y \\
    \end{cases}
\end{equation}

%

\subsection{Risk Estimators for PU Learning}
\label{sec:risknn}
Due to the absence of negative samples in PU learning, risks have to be estimated from only positive and unlabeled samples.
In other words, $R_n(\eta,f)$ need to be derived via risks from $P_p$ and $P_u=P(X)$.
Formally, a PU learning system receives training samples from $\mathsf{S} = \mathsf{P}\cup\mathsf{U}$, which can be divided into two not-necessarily-independent components.
The labeled positive samples:
$\mathsf{P} = \{x_i^p\}\sim P_p$
and unlabeled samples: 
$\mathsf{U} = \{x_i^{u}\}\sim P_u$.
Underlyingly, the unlabeled set $\mathsf{U}$ consists of positive samples $\mathsf{U_p} = \{x_i^{u_p}\}\sim P_p$ and negative samples $\mathsf{U_n} = \{x_i^{u_n}\}\sim P_n$.
As a popular practice, negative labels are assigned to the unlabeled samples~\cite{Plessis2015ConvexFF,Kiryo2017PositiveUnlabeledLW}.
Let $R_u(\eta,f,0) = \ep_{X}[\eta\ell(f,0)+(1-\eta)\ell(f,0)]$ be the risk of unlabeled samples, which are drawn from the same distribution of $P_X$, with loss function $\ell$ assigning all the labels to negative.
\textit{Unbiased PU} (aka uPU) learning methods~\cite{Plessis2014AnalysisOL,Plessis2015ConvexFF} attempt to estimate the risk for PU learning via subtracting the wrongly included risk $R_{p}(\eta,f,0)=\ep_{X}[\eta\ell(f,0)]$:
\begin{equation}
    R_{upu} := R_{p}(\eta,f,1)+R_{u}(\eta,f,0)-R_{p}(\eta,f,0)
\end{equation}
Non-negative PU (aka nnPU) ~\cite{Kiryo2017PositiveUnlabeledLW} observed that $R_{n}(\eta,f,0) = R_{u}(\eta,f,0)- R_{p}(\eta,f,0)$ should be always non-negative. 
However, this does not always hold, especially when the model $f$ becomes flexible (i.e., deep neural networks).
To mitigate this drawback, they propose a non-negative risk estimator, thus ensuring the risk will not reach negative values:
\begin{equation}\label{eq:nnpu}
    R_{nnpu} := R_{p}(\eta,f,1)+\max\Big\{0, R_{u}(\eta,f,0)- R_{p}(\eta,f,0)\Big\}.
\end{equation}

Nevertheless, minimizing these risk estimators will lead to insufficient penalty for the negative samples.
Maximizing $R_{p}(\eta,f,0)$, instead of making $R_{u}(\eta,f,0)$ small, will result in the same effect of minimizing the total risk , which is also natural side effect of minimizing $R_{p}(\eta,f,1)$ for flexible models and convex surrogate loss functions. 
Risk estimators are rectification at the outcome-of-loss-function level, which cannot avoid the explosion (i.e., in some worst case an unbounded loss may reach very large value~\cite{Kiryo2017PositiveUnlabeledLW}) of some surrogate losses, such as the popular logarithm loss.
In these cases, the flexible model overfits the training data well and sampling may include some easy positive examples which sum up to large $R_{p}(\eta,f,0)$ that overwhelms $R_{u}(\eta,f,0)$.
Thereupon, \textit{can we remedy the problem before loss function?}

\section{Collective Logarithm Loss for PU Learning}
In this section, we firstly address the decision boundary drift problem in section~\ref{sec:rectify} and provide rectification of the predictor. 
Then in section~\ref{sec:preliminary}, we introduce the background of elicitation and how it connects to the design of loss function in normal situations.
Finally, in section~\ref{sec:elicitpu}, we describe the framework of eliciting the loss function under PU learning setting.

\subsection{Rectification of Predictor}
\label{sec:rectify}
To satisfy the Fisher consistency, we hope \eqref{eq:h} hold during test, while \eqref{eq:nnpu} is biased in general~\cite{Kiryo2017PositiveUnlabeledLW}, hence leading to biased solutions.
A key observation is that the decision boundary is different between training and testing for PU learning problem. 
Our aim is to rectify the decision boundary so that the classifier for testing also fits the positive and unlabeled train data.
We introduce $\ty \in \tilde{Y}$ to denote true labels.
$y$ remains the observed labels, where unlabeled samples are regarded as negative $y=0, \forall x \in \mathsf{U}$.
Let $\eta_{e}(x)=P(\ty=1|X=x)$ be the posterior probability of testing, namely what we hope to capture by learning.
In $\mathsf{U}$, the underlying true labels $\ty=1$ for $x_i^{u_p}$ and $\ty=0$ for $x_i^{u_n}$.
It is evident that the data distribution for training and testing is different. 
Let $\eta_a$ and $\eta_e$ be the posterior probability for training and testing.
Denote by $\Omega=\mathsf{|P|+|U_p|+|U_n|}$ the total sample space, we can estimate these two expectations by the following equations in empirical estimation.

\begin{align}
    \begin{split}
        \ep_X[\eta_{a}]&=\ep_X[P_{Y|X}(1|x)] \\
        &= \frac{1}{\Omega}\sum_{x} \mathbbm{1}[x\in \mathsf{P}]
        \qquad = \frac{\mathsf{|P|}}{\Omega}
    \end{split}
    \\
    \begin{split}
        \ep_X[\eta_{e}]&=\ep_X[P_{\tilde{Y}|X}(1|x)] \\
        &= \frac{1}{\Omega}\sum_{x} \mathbbm{1}[x\in \mathsf{P}\cup\mathsf{U_p}] = \frac{\mathsf{|P|+|U_p|}}{\Omega}
    \end{split}
\end{align}
\begin{equation}\label{eq:rectify}
    \ep[\eta_e] - \ep[\eta_a] = \frac{|\mathsf{U_p}|}{\Omega}
\end{equation} 

We denote the value in Eqn.~\eqref{eq:rectify} as $\mu_p$ for the rest of the paper.
We hope $\sum_{x} \hat{\eta} \xrightarrow{p} \sum_{x}\eta_e$.
However based on PU training data, the model may converge biasedly to $\eta_a$.
Let $r=\frac{|\mathsf{U_p}|}{|\mathsf{P}|}$ be the portion of positive data in $|\mathsf{U}|$ compared to the whole P class we can derive:
\begin{equation}\label{eq:rec}
    \ep[\eta_e] = (1+r)\ep[\eta_a]
\end{equation}


\subsection{Preliminary for Elicitation}
\label{sec:preliminary}

In statistics and economics, elicitation is a practice of designing \textit{reward} mechanisms that encourage a predictor to make true \textit{predictions}. 
Let $\hat{\eta}$ be the prediction (i.e., an estimator of $\eta$), we have $\hat{\eta}=\sigma(f(x))$.
Savage et al.~\shortcite{savage1971elicitation} defines the total reward $I$ as a linear function of $\eta$.
\begin{equation}
    \label{eq:i}
    I(\hat{\eta},\eta) := \eta I_1(\hat{\eta})+ (1-\eta)I_0(\hat{\eta}),
\end{equation}
where $I_1(\hat{\eta})$ and $I_0(\hat{\eta})$ denote the \textit{conditional reward} for a certain event \textit{obtains} or \textit{not}.
In binary classification context, $I_1(\hat{\eta})$ and $I_0(\hat{\eta})$ refers to the reward for $y=1$ and $y=0$~\cite{MasnadiShirazi2008OnTD}.
Specifically, $y=1$ is regarded as the event obtains and $y=0$ otherwise.
The goal of elicitation is to design the rewards in order that a $\hat{\eta}$ maximizes $I(\hat{\eta},\eta)$ if and only if when $\hat{\eta}=\eta, \forall \hat{\eta}$.
In other words, no larger reward should be given than when prediction is ideal.
Lemma~\ref{thm:savage71} finds the sufficient and necessary condition for it.

\begin{lemma}[Savage \citeyear{savage1971elicitation}]
    \label{thm:savage71}
    Let $I(\hat{\eta},\eta)$ be as defined in \eqref{eq:i}. 
    Assume that $J(\eta):= I(\eta,\eta)$ is differentiable, then
    \begin{equation}
        \label{eq:j}
        I(\hat{\eta},\eta) \leq J(\eta), \forall \eta,
    \end{equation}
    holds and if and only if
    \begin{align}\label{eq:i1}
        I_1(\eta) &= J(\eta)+(1-\eta)J'(\eta)\\
        \label{eq:i0}
        I_0(\eta) &= J(\eta)-\eta J'(\eta)
    \end{align}    
\end{lemma}

\begin{remark}
    The equality in \eqref{eq:j} holds if and only if $\hat{\eta}=\eta$. 
    Eqn.~\eqref{eq:j} also implies that $J$ is a strictly convex function of $\eta$.
    This is the regular situation where event and prediction are in the same space.
    The event not observed will never happen, c.f., in PU learning, the event obtains even though it is not observed (i.e., unlabeled).
\end{remark}


Masnadi \textit{et al}.~\shortcite{MasnadiShirazi2008OnTD} interpreted loss functions in machine learning as a special form of $J(\eta)$.
We rewrite it in Lemma~\ref{thm:nips08} with $y\in \mathcal{Y}=\{0,1\}$ and illustrate the process of deriving the logarithm loss with Example~\ref{exa:logloss}.

\begin{lemma}[Masnadi \textit{et al}.~\citeyear{MasnadiShirazi2008OnTD}]
    \label{thm:nips08}
    
    If $J(\eta)=J(1-\eta)$, then $I_1$ and $I_0$ from \eqref{eq:i1} and \eqref{eq:i0} satisfy the following conditions. Let $c(\heta):[0,1] \mapsto [0,1]$ be an invertible link function such that $c^{-1}(v) = 1- c^{-1}(v)$.
    \begin{align}
        I_1(\heta) &= -\phi(c(\heta)) \\
        I_0(\heta) &= -\phi(c(\heta)),
    \end{align}
    Meanwhile the loss function,
    \begin{equation}
        \label{eq:phi}
        \phi(v) = J(c^{-1}(v))+(1-c^{-1}(v))J'(c^{-1}(v))
    \end{equation}
\end{lemma}

\begin{remark}
    Lemma~\ref{thm:nips08} bridges the design of a loss function $\phi(\cdot)$ with the reward function $I(\hat{\eta},\eta)$. 
\end{remark}

\begin{example}[Eliciting logarithm loss]
    \label{exa:logloss}
    Let $c$ be defined as follows,
    \begin{equation}
        \label{eq:cx}
        c(\heta) = \begin{cases}
            \hat{\eta} & \text{if } y=1 \\
            1-\hat{\eta} & \text{otherwise},
        \end{cases}
    \end{equation}
    which can be interpreted as the \textit{closeness} of the prediction to the true label.
    Intuitively, larger $c$ should get larger reward (or smaller penalty).
    Let $J(\eta) = \eta \ln \hat{\eta}(x) + (1-\eta)\ln(1- \hat{\eta})$ be the convex function. 
    Applying \eqref{eq:phi}, we can derive:
    \begin{equation}
        \begin{split}
            \phi(v) &= \left [v \ln v -\left(1-v\right) \ln\left(1-v\right)\right ]\\
                & \qquad + \left(1-v\right) \ln\big(\frac{v}{1-v}\big) \\
                &= -\ln (v) 
        \end{split}            
    \end{equation}
    The loss function is
    \begin{equation}
        \ell_{CE} = - \ln(c)
    \end{equation}
\end{example}



\subsection{Eliciting Collective Loss Function for PU Learning}
\label{sec:elicitpu}

In PU learning, the label of a specific sample in $\mathsf{U}$ is unknown. 
We only possess the statistical information of the samples.
Therefore, a reward function that suits this kind of collective information is desired.
Lemma~\ref{thm:nips08} indicates the symmetry of the link function $c(\eta)$, which changes in PU learning setting.
Let $c(x)$ be defined in \eqref{eq:cx}.
In PU learning, we must ensure $\ep[c(x)] = 1$ when $\ell(c)=0$. 
A straight-forward solution is to encourage making a certain amount of \textit{positive} predictions when the labels are negative.
The amount is such that the expectation of the predictions equals to $\mu_p$, i.e., the positive prior in unlabeled data, because $\ep[\mathbbm{1}_{\hat{y}=1}]=P(\hat{y}=1)=\mu_p$ holds. 
Under this condition, $c(x)$:
\begin{equation}    
    c(x) = \begin{cases}
        \hat{\eta}(x) & \text{if } y=1\\
        1-|\hat{\eta}(x)-\mu_p| & \text{otherwise}.
    \end{cases}
\end{equation}
Note that we apply an absolute function because when the prediction $\hat{\eta} \leq \mu_p$ it is also considered to be a negative event, thus deviates
from the correctness.
Hence, we derive the rectified reward function as follows. 
Without loss of generality, we let $\phi(\cdot)$ be logarithm function $\ln(\cdot)$.
\begin{equation}
    \label{eq:Ipu}
    I(\hat{\eta},\eta) := \eta \ln(\hat{\eta})+ (1-\eta)\ln( 1- |\hat{\eta} - \mu_p| ).
\end{equation}
According to \cite[section 7]{savage1971elicitation}, \eqref{eq:Ipu} must be upper bounded by a maximum reward. 
We further detail it in Theorem~\ref{thm:elicit}.
\begin{theorem}[Maximum reward in PU learning]
    \label{thm:elicit}
    Let $I(\hat{\eta},\eta)$ be defined as in \eqref{eq:Ipu}. 
    There exists convex function:
    \begin{equation}
        \label{eq:Jpu}
        \resizebox{.98\linewidth}{!}{$
        \displaystyle
        J(\eta) = 
        \begin{cases}
            \eta \ln \big( \eta(1+\mu_p) \big) + (1-\eta)\ln \big( (1-\eta)(1+\mu_p) \big) & \text{if } \eta>\mu_p ;\\
            \eta \ln ( \mu_p) & \text{otherwise,} 
        \end{cases}
        $}
    \end{equation}
    such that the reward function \eqref{eq:Ipu} supports $J(\eta)$ at and only at $\hat{\eta} = \eta(1+\mu_p)$.
    
\end{theorem}

\begin{proof}
     
     Consider $\eta$ to be fixed at a constant value $k$. At this point, 
     \begin{itemize}
         \item if $\heta>\mu_p$, $I$ is a concave function of $\hat{\eta}$ that reaches maximum when $\frac{\partial I}{\partial \hat{\eta}}=0$.
         That is,      
         \begin{gather}
            \nonumber
            \frac{k}{\heta}- \frac{1-k}{1-|\heta-\mu_p|}\frac{\heta-\mu_p}{|\heta-\mu_p|} =0\\
            \label{eq:heta1}
            \heta = k(1+\mu_p)    
         \end{gather}
         \item if $\heta \leq \mu_p$, $I$ is monotonically increasing and reaches maximum when $\mu_p$ is at maximum
         \begin{equation}\label{eq:heta2}
            \heta=\mu_p    
         \end{equation}         
     \end{itemize}
     Plugging the value $\heta$ in \eqref{eq:heta1} and \eqref{eq:heta2} into \eqref{eq:Ipu} will derive ~\eqref{eq:Jpu}.
\end{proof}

\subsection{Implementation}

We apply stochastic gradient optimization.
Instead of traditional one-loss-per-sample paradigm, we collect the model predictions from multiple samples while update the gradient only once.
That is equivalent to ask multiple agents to make decisions under condition of \eqref{eq:Ipu}
The intuition is as follows:
It is difficult to ensure the correctness of a single prediction especially under unlabeled data. 
The underlying label may be either positive or negative.
However, when a batch of samples are considered together, the expectation of the prediction converges to $\mu_p$.
For all mini-batch $\mathsf{S_b}=\mathsf{P_b}\cup\mathsf{U_b}$, the loss function is as follows.
\begin{equation}
    \label{eq:loss}
    \begin{split}
        \ell(\heta,y) &= \begin{cases}
                         -\ln \heta  &\text{if $y=1$} \\
                         -\ln (1-|\frac{1}{|\mathsf{U_b}|}\sum_{x \in \mathsf{U_b}}\heta - \mu_p |) &\text{if $y=0 $.}
                     \end{cases}
    \end{split}
\end{equation}




In practice, we treat the positive class prior $\mu_p$ in $\mathsf{U}$ as known during training. 
Many related works~\cite{Plessis2015ClasspriorEF,Bekker2018EstimatingTC} can be applied to estimate it.
We further show that our method is insensitive to it in Section~\ref{sec:robust}.




%
        


\subsection{Estimation error bound}
We next theoretically upper bound the generalization error.
Let $\heta_{pu}$ be the empirical risk minimizer corresponding to~\eqref{eq:loss}.
The learning problem is to find an optimal decision function $\heta^*$ in the function class $\mathcal{F}=\{\heta \mid \norm{\heta}_{\infty}<C_{\eta} \}$ where $C_{\eta}$ is a constant. 
Formally, $\heta^* = \argmin_{\heta \in \mathcal{F}} R(\heta)$.
Let $\mathfrak{R}_n$ be the Rademacher complexity defined in~\cite{Bartlett2001RademacherAG}.
\begin{lemma}[\citeauthor{Ledoux1991ProbabilityIB} \citeyear{Ledoux1991ProbabilityIB}]
    \label{thm:talagrand} Assuming $\Phi: \mathbb{R}\mapsto\mathbb{R}$ is Lipschitz continuous with constant $L_{\Phi}$ and $\Phi(0) = 0$, we have
    \begin{equation}
        \mathfrak{R}_n(\Phi\circ \mathcal{F} ) \leq L_\Phi \mathfrak{R}_n(\mathcal{F}).
    \end{equation}    
\end{lemma}

\begin{theorem}[Generalization error bound]
    \label{thm:Rademacher}

    For any $ \epsilon > 0$, with probability at least $1-\epsilon$:
    \begin{equation}
        \begin{split}            
            R(\heta_{pu})-R(\heta^*) &\leq 2 \mathfrak{R}_{n}(\Phi\circ \mathcal{F}  ) +\sqrt{\frac{2\log(2/\epsilon)}{n}}\\
            & \leq \frac{C_x}{m} \mathfrak{R}_{n}(\mathcal{F}) +\sqrt{\frac{2\log(2/\epsilon)}{n}},
        \end{split}
    \end{equation}
    where $n$ is the total number of i.i.d. samples corresponding to the Rademacher variables.
\end{theorem}

The Lipschitz constant is $\frac{\norm{X}_2}{2m} \leq \frac{C_x}{2m} $ for original cross entropy~\cite{yedida} where $m=\frac{|\Omega|}{|\mathsf{U_b}|}$ and $\norm{X}_2 \leq C_x, C_x\in \mathbb{R}^+$ is the input vector norm.
This Lipschitz constant also applies for \eqref{eq:loss}, so that the last inequality follows.
The penultimate inequality follows from routine proof of generalization bound using Rademacher complexity~\cite[Section 26.1]{ShalevShwartz2014UnderstandingML}.

\begin{table*}[h!]
    \centering
\begin{tabular}{cccccc}
        \toprule
        Dataset  & \#Train  & \#Test   & Details  & P class  & N class \\
        \midrule
        MNIST    & 60000    & 23878    & 32$\times$32 image & 0,2,4,6,8 & 1,3,5,7,9 \\[2pt]
        USPS     & 7291     & 2942     & 32$\times$32 image & 0        & rest \\[2pt]
        SVHN     & 73257    & 20718    & 16$\times$16 image & 1,2,3,4,5 & 6,7,8,9,0 \\[2pt]
        \multirow{2}[0]{*}{CIFAR-10} & \multirow{2}[0]{*}{50000} & \multirow{2}[0]{*}{19947} & \multirow{2}[0]{*}{32$\times$32 image} & ‘bird’, ‘cat’, ‘deer’,  & ‘airplane’, ‘auto mobile’,  \\
                 &          &          &          & ‘dog’, ‘frog’, and ‘horse’ & ‘ship’, and ‘truck’ \\[2pt]
        \multirow{2}[1]{*}{20ng} & \multirow{2}[1]{*}{11314} & \multirow{2}[1]{*}{7532} & \multirow{2}[1]{*}{text} & ‘alt.’, ‘comp.’, & ‘sci.’, ‘soc.’ \\
                 &          &          &          &  ‘misc.’ and ‘rec.’ &  and ‘talk.’ \\
        \bottomrule
\end{tabular}%
    \caption{Dataset specification. The last two rows elaborate how positive and negative classes are formed.}
    \label{tab:dataspec}%
\end{table*}%

\section{Experiments}
We perform experiments on five real-world datasets, including MNIST~\cite{lecun1998gradient}, USPS~\cite{hastie2005elements}, SVHN~\cite{netzer2011reading}, CIFAR-10~\cite{Krizhevsky2009LearningML} and 20ng (twenty news groups)~\cite{Lang1995Newsweeder}.
We choose the positive and negative class in accordance with the previous research~\cite{Kiryo2017PositiveUnlabeledLW}.
The specification of datasets are described in Table~\ref{tab:dataspec}.
We still need the actual label for testing the models, hence we use originally labeled data.
Specifically, we randomly pick $r=20\%, 30\%,40\%, 80\%$ of P class data and mix them with all the N class data to compose the unlabeled set $\mathsf{U}$.
The remaining P class data forms the positive set $\mathsf{P}$.

We apply neural networks as the predictor function. 
Specifically, we apply vanilla vgg-16 structure~\cite{Simonyan2014VeryDC} to encode the input features.
For 20ng, all the details including model structure (a multilayer perceptron with five layers and the activation functions are \texttt{Softsign}) and pre-trained word embedding (300-dimension GloVe~\cite{Pennington2014GloveGV} word embeddings) are same with ~\cite{Kiryo2017PositiveUnlabeledLW}.
For the optimizer, we use Nadam~\cite{dozat2016incorporating} with learning rate 0.0005 throughout all models.
The parameters in nnPU are set equal to the original paper, i.e., $\beta=0,\gamma=1$.

We then evaluate the results to show the efficacy of proposed method cPU.
We explore the following two common questions in applications:
1) Can it separate the unlabeled positive samples from the negative ones without explicit exposure to negative samples?
2) Is it sensitive to class prior, which may vary and sometimes with uncertainty in real applications?

\begin{table*}[htb]
    \centering
    \small    
\begin{tabular}{ccccccc}
    \toprule
    Dataset  & r        & uPU      & nnPU     & LDCE     & PULD     & cPU (ours) \\
    \midrule
    \multirow{4}[1]{*}{MNIST} & 0.2      & 0.9920 $\pm$ 0.0003 & 0.9868 $\pm$ 0.0011 & --       & --       & \boldmath{}\textbf{0.9925 $\pm$ 0.0003}\unboldmath{} \\
             & 0.3      & 0.9910 $\pm$ 0.0006 & 0.9859 $\pm$ 0.0010 & --       & --       & \boldmath{}\textbf{0.9911 $\pm$ 0.0002}\unboldmath{} \\
             & 0.4      & 0.9898 $\pm$ 0.0005 & 0.9853 $\pm$ 0.0011 & --       & --       & \boldmath{}\textbf{0.9907 $\pm$ 0.0008}\unboldmath{} \\
             & 0.8      & 0.9772 $\pm$ 0.0013 & 0.9787 $\pm$ 0.0005 & --       & --       & \boldmath{}\textbf{0.9851 $\pm$ 0.0006}\unboldmath{} \\[2pt]
    \multirow{4}[0]{*}{USPS} & 0.2      & 0.9396 $\pm$ 0.0015 & \boldmath{}\textbf{0.9624 $\pm$ 0.0030}\unboldmath{} & 0.934    & --       & 0.9606 $\pm$ 0.0009 \\
             & 0.3      & 0.9398 $\pm$ 0.0024 & \boldmath{}\textbf{0.9638 $\pm$ 0.0034}\unboldmath{} & 0.911    & --       & 0.9599 $\pm$ 0.0027 \\
             & 0.4      & 0.9357 $\pm$ 0.0046 & 0.9595 $\pm$ 0.0017 & 0.901    & --       & \boldmath{}\textbf{0.9624 $\pm$ 0.0017}\unboldmath{} \\
             & 0.8      & 0.9334 $\pm$ 0.0031 & 0.9316 $\pm$ 0.0077 & --       & --       & \boldmath{}\textbf{0.9501 $\pm$ 0.0018}\unboldmath{} \\[2pt]
    \multirow{4}[0]{*}{SVHN} & 0.2      & 0.9082 $\pm$ 0.0023 & 0.8972 $\pm$ 0.0036 & 0.785    & 0.851    & \boldmath{}\textbf{0.9150 $\pm$ 0.0014}\unboldmath{} \\
             & 0.3      & 0.9044 $\pm$ 0.0017 & 0.8995 $\pm$ 0.0021 & 0.776    & 0.852    & \boldmath{}\textbf{0.9102 $\pm$ 0.0020}\unboldmath{} \\
             & 0.4      & 0.9027 $\pm$ 0.0022 & 0.8953 $\pm$ 0.0037 & 0.748    & 0.850    & \boldmath{}\textbf{0.9083 $\pm$ 0.0023}\unboldmath{} \\
             & 0.8      & \boldmath{}\textbf{0.8679 $\pm$ 0.0039}\unboldmath{} & 0.8569 $\pm$ 0.0049 & --       & --       & 0.8595 $\pm$ 0.0019 \\[2pt]
    \multirow{4}[0]{*}{CIFAR-10} & 0.2      & 0.8534 $\pm$ 0.0032 & 0.8374 $\pm$ 0.0033 & 0.772    & 0.834    & \boldmath{}\textbf{0.8610 $\pm$ 0.0029}\unboldmath{} \\
             & 0.3      & 0.8427 $\pm$ 0.0024 & 0.8264 $\pm$ 0.0056 & 0.761    & 0.861    & \boldmath{}\textbf{0.8556 $\pm$ 0.0054}\unboldmath{} \\
             & 0.4      & 0.8351 $\pm$ 0.0049 & 0.8178 $\pm$ 0.0063 & 0.701    & 0.860    & \boldmath{}\textbf{0.8446 $\pm$ 0.0038}\unboldmath{} \\
             & 0.8      & 0.7636 $\pm$ 0.0025 & 0.7494 $\pm$ 0.0023 & --       & --       & \boldmath{}\textbf{0.7906 $\pm$ 0.0021}\unboldmath{} \\[2pt]
    \multirow{4}[1]{*}{20ng} & 0.2      & 0.8601 $\pm$ 0.0013 & 0.7675 $\pm$ 0.0410 &   --     &   --    & \boldmath{}\textbf{0.8601 $\pm$ 0.0012}\unboldmath{} \\
             & 0.3      & 0.8589 $\pm$ 0.0014 & 0.8132 $\pm$ 0.0180 &   --      &  --      & \boldmath{}\textbf{0.8599 $\pm$ 0.0034}\unboldmath{} \\
             & 0.4      & 0.8573 $\pm$ 0.0050 & 0.8414 $\pm$ 0.0047 &  --     &  --     & \boldmath{}\textbf{0.8592 $\pm$ 0.0041}\unboldmath{} \\
             & 0.8      & 0.8422 $\pm$ 0.0027 & 0.8191 $\pm$ 0.0022 & --       &   --     & \boldmath{}\textbf{0.8428 $\pm$ 0.0028}\unboldmath{} \\
    \bottomrule
    \end{tabular}%
    \caption{Comparison with current state-of-the-art in accuracies. Each experiment is repeated five times. The reported values are in the format of ``\textit{mean} $\pm$ \textit{standard deviation}". The results of LDCE and PULD are excerpted from the original paper, thus without standard deviation values. }
    \label{tab:expall}%
\end{table*}%

\subsection{Comparison to State of the Art}
We first show the overall evaluation results on the real-world datasets.
We compare our proposed approach with current state-of-the-art PU
learning methods: unbiased PU (uPU) \cite{Plessis2015ConvexFF} and non-negative PU (nnPU) \cite{Kiryo2017PositiveUnlabeledLW}, LDCE \cite{Shi2018PositiveAU} and PULD \cite{Zhang2019PositiveUnlabeledLW}.
We re-implement uPU and nnPU using the same vgg-16 structure as in our method.
We do not compare with LDCE and PULD, but simply provide the results for reference because: 
1) they require additional features construction/engineering process, which is not explicit;
2) these two models deeply involve support vector machine~\cite{Cortes1995SupportVectorN} as their model, and thus can neither be plugged in by other loss functions than hinge loss nor be fairly compared with neural networks.
The experiments are repeated five times with randomly sampled P class each. 
We report the mean and standard deviation of accuracies in Table~\ref{tab:expall}.
We can see that our proposed method cPU outperforms the current state-of-the-art methods in most cases and are relatively more stable (smaller standard deviation).
On the rather more difficult dataset CIFAR-10, cPU achieves a healthy 1-4 point accuracy gap with the closest competitor.
Note that, Zhang \textit{et al}.~\shortcite{Zhang2019PositiveUnlabeledLW} reported that nnPU performs dramatically worse than other competitors (i.e., the best accuracy is 0.771 for CIFAR-10 at $r=30\%$), which did not happen in our experiments. 

\subsection{Robustness}
\label{sec:robust}
\begin{table}[htb]
    \centering
    \small
\begin{tabular}{ccrrrr}
    \toprule
    \multirow{2}[4]{*}{Dataset} & \multirow{2}[4]{*}{r} & \multicolumn{4}{c}{$\Delta \mu_p$} \\
    \cmidrule{3-6}         &          & -10\%    & -5\%     & +5\%     & +10\% \\
    \midrule
    \multirow{4}[1]{*}{MNIST} & 0.2      & 0.9925   & 0.9924   & 0.9927   & 0.9842 \\
             & 0.3      & 0.9905   & 0.9912   & 0.9911   & 0.9878 \\
             & 0.4      & 0.9908   & 0.9907   & 0.9907   & 0.9881 \\
             & 0.8      & 0.9832   & 0.9842   & 0.9855   & 0.9609 \\[3pt]
    \multirow{4}[0]{*}{USPS} & 0.2      & 0.9651   & 0.9616   & 0.9621   & 0.9542 \\
             & 0.3      & 0.9641   & 0.9581   & 0.9656   & 0.9517 \\
             & 0.4      & 0.9631   & 0.9631   & 0.9606   & 0.9527 \\
             & 0.8      & 0.9601   & 0.9562   & 0.9283   & 0.9128 \\[3pt]
    \multirow{4}[0]{*}{SVHN} & 0.2      & 0.9156   & 0.9078   & 0.9124   & 0.9022 \\
             & 0.3      & 0.9106   & 0.9159   & 0.9097   & 0.8921 \\
             & 0.4      & 0.9085   & 0.9096   & 0.8989   & 0.8888 \\
             & 0.8      & 0.8763   & 0.8755   & 0.8402   & 0.8283 \\[3pt]
    \multirow{4}[0]{*}{CIFAR-10} & 0.2      & 0.8597   & 0.8607   & 0.8609   & 0.8543 \\
             & 0.3      & 0.8542   & 0.8527   & 0.8519   & 0.8549 \\
             & 0.4      & 0.8343   & 0.8401   & 0.8435   & 0.8311 \\
             & 0.8      & 0.7778   & 0.7854   & 0.7854   & 0.7797 \\[3pt]
    \multirow{4}[1]{*}{20ng} & 0.2      & 0.8593   & 0.8594   & 0.8606   & 0.8598 \\
             & 0.3      & 0.8589   & 0.8598   & 0.8602   & 0.8596 \\
             & 0.4      & 0.8590   & 0.8605   & 0.8578   & 0.8563 \\
             & 0.8      & 0.8417   & 0.8421   & 0.8426   & 0.8352 \\
    \bottomrule
    \end{tabular}%
    
    \caption{Demonstrating the robustness of our method. $\mu_p$ is adjusted slightly lower/higher for in each column to test whether the accuracy result is sensitive to some imprecise $\mu_p$.}
    \label{tab:robust}%
\end{table}%
  
In this section, we study a common scenario of PU-learning in which the class prior is not accurately estimated.
This usually happens in real applications, where a small sample can be achieved to approximate the class prior $\pi$.
To simulate the scenario, we set $r=20\%,30\%,40\% $ and misspecify $\mu_p$.
The results are shown in Table~\ref{tab:robust}.
We can observe that generally the results are worse if deviation of $\mu_p$ become big.
Another phenomenon is that, the bigger $r$ , the deviation are more influential to the results. 
This can be avoided by sampling more data to get better estimation of $\mu_p$, since larger $r$ indicates more unlabeled data available in real applications.
Nevertheless, the fluctuation is acceptable when $\mu_p$ varies, which means our proposed approach is robust towards wrongly estimated prior probabilities of P class in unlabeled data.

\begin{figure*}[h!]
    \centering
    \includegraphics[width=0.95\linewidth]{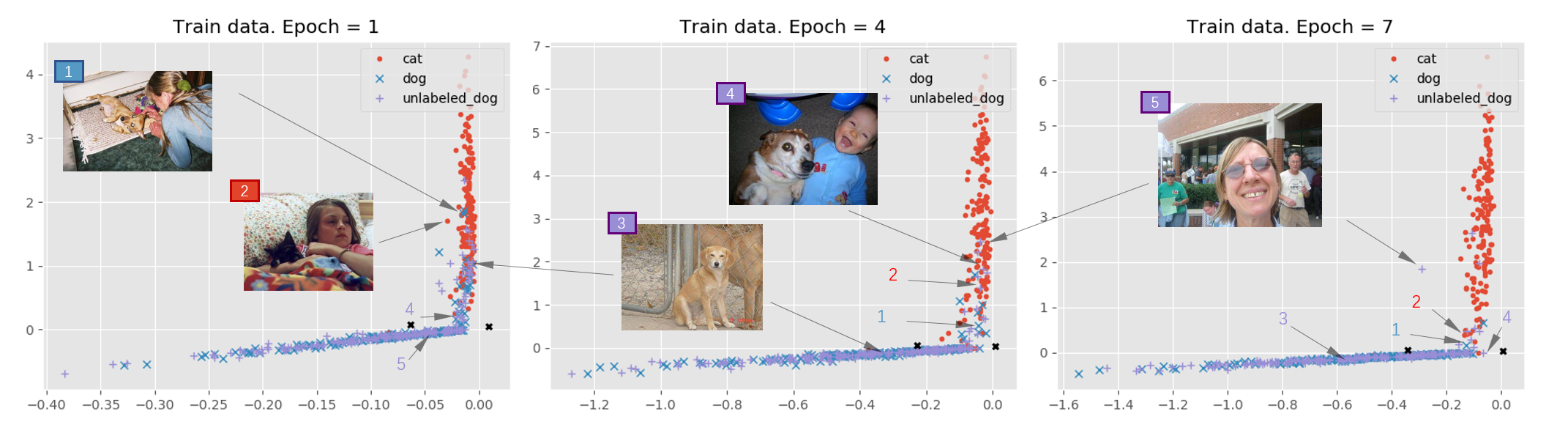}
    \caption{Training process visualization. The feature space contains data points from the 1st, 4th and 7th epoch. Some interesting examples (e.g., photograph with human) in data are shown from sample 1 to 5 (denoted S1...S5 for short). S1 is a labeled dog. S2 is a cat (unlabeled). The rest are dogs (unlabeled). Note that, cat and unlabeled dogs are unlabeled in the view of model.}
    \label{fig:trainvis}
\end{figure*}

\subsection{Training Process Analysis and Case Study}
In order to get a deeper insight on how loss function take effect, we project the layer before last onto a 2D fully connected layer and plot its activation~\footnote{We uniformly sample 500 examples from the training set for clarity of plot.}.
For simplicity, we demonstrate with a toy dataset DVC (dog-vs-cat)~\footnote{\url{https://www.kaggle.com/c/dogs-vs-cats/data}}, in which $r=40\%$ dogs are mixed with cats to form the unlabeled. 
A snapshot of the 1st, 4th and 7th epoch is shown in Figure~\ref{fig:trainvis} along with five samples (denoted S1...S5). 
We observe that some unlabeled dogs are blended with the cats at first.
As the training progressed, they gradually move towards the positive dogs.
S3 is a typical example.
This further supports the assertion that unlabeled positive samples are separable even without explicit negative examples.
We also analyze the errors. 
S1 and S2 are special examples with human inside.
We observe S1 is guided by positive label and move towards the positive center, while S2, in which the cat is barely recognizable, move towards S1, due to their resemblance, and lead to a wrong prediction.
S4 and S5 are noisy unlabeled samples.
As a result, they move back and forth across the borderline. 
This might be a useful signal for active learning, which will be left for future works.

\section{Conclusion}
In this paper, we identify the bias caused by class uncertainty in the unlabeled as the major difficulty for current risk estimators.
We propose a novel approach towards PU learning dubbed ``cPU" that collectively process the predictions.
We design the loss function through theoretical elicitation PU learning setting and rectification of the predictor.
It outperforms the state-of-the-art methods on PU learning and shows robustness against wrongly estimated class prior on the unlabeled data.

\bibliographystyle{named}
\bibliography{longrefer}



\end{document}